\documentclass[10pt,twocolumn,letterpaper]{article}

\usepackage{cvpr}
\usepackage{times}
\usepackage{epsfig}
\usepackage{epstopdf}
\DeclareGraphicsExtensions{.eps}
\usepackage{graphicx}
\usepackage{relsize}
\usepackage{amsmath}
\usepackage{amssymb}
\usepackage{amsthm}
\usepackage{bbm}
\usepackage{mathtools}
\newcommand{\mytilde}{\raise.17ex\hbox{$\scriptstyle\mathtt{\sim}$}}
\newcommand{\argmin}{\mathop{\mathrm{argmin}}}

\usepackage{algorithm}
\usepackage[noend]{algpseudocode}
\algnewcommand\algorithmicinput{\textbf{Input:}}
\algnewcommand\INPUT{\item[\algorithmicinput]}
\algnewcommand\algorithmicoutput{\textbf{Output:}}
\algnewcommand\OUTPUT{\item[\algorithmicoutput]}
\algdef{SE}[DOWHILE]{Do}{doWhile}{\algorithmicdo}[1]{\algorithmicwhile\ #1}

\newtheorem{theorem}{Theorem}
\newtheorem{lemma}[theorem]{Lemma}

\newcommand{\wl}{W_{\textsc{\textsc{\relsize{-2}{\textsl{LHM}}}}}}
\newcommand{\wkl}{W_{\textsc{\textsc{\relsize{-2}{\textsl{KHHM}}}}}}
\usepackage[pagebackref=true,breaklinks=true,letterpaper=true,colorlinks,bookmarks=false]{hyperref}

 \cvprfinalcopy % *** Uncomment this line for the final submission

\def\cvprPaperID{1975} % *** Enter the CVPR Paper ID here

\ifcvprfinal\fi
\begin{document}

%%%%%%%%% TITLE
\title{Latent Hinge-Minimax Risk Minimization for
%Imbalanced Binary Training and Multi-Class
Inference from a Small Number of Training Samples}

\author{Dolev Raviv\\
University of Haifa\\
{\tt\small dolev.raviv@gmail.com}
\and
Margarita Osadchy\\
University of Haifa\\
{\tt\small rita@cs.haifa.ac.il}
}

\maketitle
%\thispagestyle{empty}

%%%%%%%%% ABSTRACT
\begin{abstract}
Deep Learning (DL) methods show very good performance when trained on large, balanced data sets. However, many practical problems involve imbalanced data sets, or/and classes with a small number of training samples. The performance of DL methods as well as more traditional classifiers drops significantly in such settings. Most of the existing solutions for imbalanced problems focus on customizing the data for training. A more principled solution is to use mixed \emph{Hinge-Minimax risk}~\cite{osadchy2015k} specifically designed  to solve binary problems with imbalanced training sets. Here we propose a Latent Hinge Minimax (LHM) risk and a training algorithm that generalizes this paradigm to an ensemble of hyperplanes that can form arbitrary complex, piecewise linear boundaries. To extract good features, we combine LHM model with CNN via transfer learning. To solve multi-class problem we map pre-trained category-specific LHM classifiers to a multi-class neural network and adjust the weights with very fast tuning. LHM classifier enables the use of unlabeled data in its training and the mapping allows for multi-class inference, resulting in a classifier that performs better than alternatives when trained on a small number of training samples.

\end{abstract}

%%%%%%%%% BODY TEXT
\section{Introduction}

Many real binary classification problems involve imbalanced classes, for example object detection in vision and fraud detection in security.  In such problems it is easy to collect background data, while data representing the target class is rare or hard (expensive) to obtain. The majority of existing powerful classifiers (e.g., SVM, Neural Networks, including deep ones) assume balanced training sets and when trained on imbalanced sets show degraded classification performance.

Deep Neural Networks have recently shown very impressive performance in large-scale multi-class problems \cite{krizhevsky2012imagenet, taigman2014deepface, SimonyanZ14a, SzegedyLJSRAEVR15}. However, these models require very large number of \textbf{labeled} training samples and their performance drops rapidly when the training set size gets smaller. Note that the requirement of large labeled sets is expensive in terms of data collection and training time. In practice, many learning problems require rapid inference from small amounts of data.

The aim of this work is to develop powerful and fast classifiers that improve in both tasks: 1) training in imbalanced setting that involve a small number of positive training samples and a large number of negative data points; 2) multi-class problems with a small number of labeled samples.

We follow the paradigm introduced in~\cite{osadchy2012hybrid}  that combines hinge risk for the smaller class and minimax risk~\cite{lanckriet2002robust, honorio2014tight} for the larger class to address imbalanced classification problems. The mixed risk was used to train linear and kernel hybrid classifiers in~\cite{osadchy2012hybrid}.  Unfortunately, while being well understood and fast, linear classifiers do not solve all machine learning problems. Kernel methods show good classification results for highly non-linear problems, but they suffer from long running time and are not scalable to large tasks. To address these issues, \cite{osadchy2015k} derived a  hinge-minimax risk and an efficient training algorithm for intersection of $K$ positive halfspaces. Such an intersection forms a convex set which limits the applicability of the classifier.

In this work, we generalize the hinge-minimax risk for an ensemble of linear classifiers, that can form arbitrary, piece-wise linear boundaries.  We propose a training algorithm that minimizes this risk by simultaneously discovering the convex components in the positive class and building K-hyperplane models to separate each component from the negative class. The learning is done by alternating between finding the best partition of the data into hidden components and updating the model over this partition. We call our novel classifier the \emph{Latent Hinge Minimax} (LHM) classifier, as it discovers the latent structure in the data and employs the Hinge-Minimax paradigm.

We show that in imbalanced setting the proposed LHM classifier outperforms other combinations of hyperplanes, including Neural Network (NN) with an equivalent architecture (NN can be viewed as a combination of hyperplanes). The robustness of LHM to imbalanced problems can be explained by the use of the minimax risk~\cite{lanckriet2002robust, honorio2014tight}, that serves as a regularizer in training (since it utilizes the statistics of the entire class, as opposed to learning from small batches of examples).

The LHM training procedure is designed for binary problems. To apply it in a multi-class setting, we build one-against-all classifiers for all classes and combine them in a single model by mapping class specific LHM models to a multi-class NN with a matching architecture (see Section~\ref{sec:mc_nn}). We then use the cross-entropy loss to adjust the weights in the resulting \emph{LHM-NN} combination.

To solve classification problems with a small number of training examples, it was suggested (e.g.,\cite{DonahueJVHZTD14,ROSS,Littwin_2016_CVPR}) to combine a pre-trained CNN (trained on a much larger training set for a related classification problem) for feature extraction, with a classifier for the target problem. Such an approach was also referred to as \emph{transfer learning}.
If the classifier is implemented as a neural network, it enables an end-to-end training, which usually improves the results. We show that using LHM-NN in the transfer learning settings has significant benefits compared to NN, in both classification accuracy and training efficiency. The improved accuracy stems from the ability of LHM model to learn from unlabeled data. The fast convergence of the LHM-NN (just a handful of epochs) is due to a very good initialization of the upper layers with class specific LHM classifiers. Note that class specific LHM models can be trained in parallel while a distributed training of fully connected layers in NN is far from being trivial. Moreover, adding a new class to LHM-NN is fast and easy: train a classifier for the new class, map it to the corresponding LHM-NN architecture and run a very fast fine-tuning. Similarly to \cite{fromNtoN+1}, which considered the transfer learning for the $n+1$ category from a fully trained $n$-category classifier, we use only a handful of training samples for tuning it. In contrast to \cite{fromNtoN+1}, we do not restrict the new classifier to belong to the span of the previously learned $n$ classifiers. This allows us greater flexibility in adding a new, non-related class to the multi-class model.

The method proposed here is different from the one-shot learning approach~\cite{Koch2015,SantoroBBWL16}, which attempts to find a mapping between target and source examples and apply it to the examples or to the model. LHM classifier learns the target concept from its examples, leveraging from unlabeled data in modeling the background statistics.

\section{Background}
We first address the settings in which the positive labeled class is much smaller then the negative class. It was shown that hinge loss \cite{vapnik2000nature, zhang2002covering, bartlett2002rademacher, bousquet2004introduction, kakade2009complexity} is computationally appealing when there are fairly small number of training samples, thus it could be used to measure the positive class risk within imbalanced problem settings.  Alternatively, the minimax risk~\cite{lanckriet2002robust, honorio2014tight} upper bounds the distribution that generates the instances-labels examples in the world. This approach is computationally appealing when there are (infinitely) many training examples, since it only utilizes their statistical properties, such as mean and covariance. Consequently, it could be employed as the negative class risk.

In this work we derive a  mixed risk and an efficient training algorithm for a more general ensemble of  hyperplanes. Our approach builds upon the mixed risk for the intersection of K-hyperplanes~\cite{osadchy2015k} which is briefly summarized in section~\ref{KHHM}.

\subsection{K-hyperplane Hinge-Minmax Classifier} \label{KHHM}
Let $(x,y)\mytilde D$  be a joint distribution of samples $x \in \mathbb{R}^n$ and labels $ y \in \{-1,1\}$.  Let $D_{neg}$ be a marginal distribution of samples over the negative labels, and $\mu$ and $\Sigma$ be its mean and covariance respectively.  For simplicity, unless stated otherwise, for a linear classifier $w$ which predicts $y=sign(w^Tx)$, we assume that $b=0$ (or absorbed by $w$).

Let $w_j$, $j=1,..,K$ denote $K$ hyperplanes. Let $W$ be a $K\times d$ matrix with $w_j$ as its $j$th column. A K-hyperplane Hinge-Minmax classifier (KHHM) is an intersection of positive half-spaces defined by these $K$ hyperplanes.

Let $X^+ \triangleq \{x\in X: y = 1\}$, and $X^-\triangleq \{x\in X: y = -1\}$ denote the positive and negative training sets correspondingly and let $m^+$ be the size of $X^+$ and $m^-$ be the size of $X^-$.
Let $\hat{\mu}$ and $\hat{\Sigma}$ be the mean and covariance of $D_{neg}$, estimated using $X^-$. The KHHM training algorithm in~\cite{osadchy2015k} minimizes the empirical risk:
\begin{equation}\label{eq:KHHM_risk}
%L^{MH}(W)=L^{M,-1}_{X^-}(W)+L^{H,1}_{X^+}(W)
L(\wkl)=L^{M,-1}_{X^-}(W)+L^{H,1}_{X^+}(W)
\end{equation}
where $L^{M,-1}_{X^-}(W)=\sup_{z\sim Z(\hat{\mu},\hat{\Sigma})} \Pr(z \in Q)$ is the minimax risk over the negative labels inside the intersection $Q\triangleq \{x:{W}^Tx \geq \vec{0}$ (the zero vector)$\}$. It was shown in~\cite{osadchy2015k}  that $$\sup_{z\sim Z(\mu, \Sigma)}\Pr(W^Tz>\vec{0})=\frac{1}{1+d^2}$$
with $d^2=\mu^T\tilde{W}(\tilde{W}^T \Sigma\tilde{ W})^{-1}\tilde{W}^T\mu,$ where
$\tilde{W}$ is a sub-matrix of $W$ containing hyperplanes that intersect in a point closest to $\mu$ scaled by $\Sigma^{-1}$.

The hinge part of the risk in Eq.~\ref{eq:KHHM_risk} is defined as $L^{H}_{X^+}(W)=\sum_{x \in X^+} \ell(W;x,1)$, where $\ell(W;x,y)=\sum_j{\max{\{0,1-yw_j^Tx\}}}$ is the K-hyperplane hinge loss~\cite{osadchy2015k}.

\section{Latent Hinge-Minmax Classifier}
To accommodate classes that form non-convex or disjoint sets, we propose a new model, Latent Hinge-Minmax (LHM) classifier, and a training scheme that simultaneously discovers the convex components in the positive class and learns the K-hyperlane models separating each convex component from the negative class.

We define the LHM classifier as a union of intersections of positive half-spaces. We assume that each intersection is composed of $K$ hyperplanes: $W^i=[w^i_1,...,w^i_K]$ and there are $C$ components in the union. Let $\wl \triangleq (W^1,\dots,W^C)$ denote the LHM model. Equivalently, we can define the LHM classifier as
$$f_{\textsc{LHM}}(x,\wl)=sign(\max_{i\in \{1..C\}}\{{\min_{j\in \{1..K\}}{{w^{i}_j}}^Tx\}}).$$

\subsection{Latent Hinge-Minimax Risk}
We extend the hinge-minimax risk in Eq.~\ref{eq:KHHM_risk}, to contain multiple latent components and a hidden assignment variable. Specifically,  we define a latent variable $\varphi(x)=i, \;\; i \in \{0,\ldots,C\}$ for each sample $(x,y)\in D$. We set $\varphi(x)=0$ for all samples with the negative label.  Since the assignment of negative training samples  is constant during the training, we reduce the set of latent values to $\{1,\ldots,C\}$ for the simplicity of notation.

\noindent We define the LHM risk function as follows:
\begin{align} \label{eq:lhm_risk}
L_D(\wl;\varphi)=L^{M}_{\mu,\Sigma}(\wl;\varphi)+L^{H}_D(\wl;\varphi),
\end{align}
where
\begin{equation}\label{eq:minimax_expected}
 L^{M}_{\mu,\Sigma}(\wl;\varphi)=\Pr_{z\sim Z(\mu,\Sigma)}(z \in \bigcup_{i\in\{1..C\}}Q^i)
\end{equation}
is the minimax part of the LHM risk and
\begin{flalign}\label{eq_expected_hinge_inLHM}
&L^H_D(\wl;\varphi)= \\ &\mathbb{E}_{(x,y)\in D}\left[\sum_{i=1}^C\ell(W^i;x,y)\mathbbm{1}\left[\varphi(x)=i\right]\right]&& \nonumber
\end{flalign}
is the hinge part, where $$\ell(W;x,y) = \max_{j\in \{1..K\}}\{\max\{0,\alpha-yw_j^Tx\}\}$$
is the modified K-hyperplane hinge loss. This change is required to accommodate comparison between the different norms of the hyperplanes.

\subsection{Empirical Risk}
Each sample with positive label encounters a loss  only in a single latent component, specified by its latent variable $\varphi(x)$ as per Eq.~\ref{eq_expected_hinge_inLHM}. Thus,we define a single positive sample loss as follows,
\begin{flalign*}
&L(W^{\varphi(x)};x,1,\varphi(x))\\ \nonumber
&=\frac{1}{m^+_{\varphi(x)}} L^M_{X^{\textsc{\textsc{\relsize{-2}{\textsl{-}}}}}}(W^{\varphi(x)})+\lambda L^H_{X^{\textsc{\textsc{\relsize{-2}{\textsl{+}}}}}}(W^{\varphi(x)};x)&&
\end{flalign*}
where
\small $L^M_{X^\textsc{-}}(W^{\varphi(x)})= {\sup_{z\sim  Z(\hat{\mu},\hat{\Sigma})} \Pr(z \in Q^i)}$
\normalsize is constant for all positive examples with the same assignment (the mean and covariance are estimated from $X^-$) and $L^H_{X^\textsc{+}}(W^{\varphi(x)};x) = \ell(W^{\varphi(x)};x,1).$

Let $X^i \triangleq \{x\in X^\textsc{+} : \varphi(x) = i\}$ define a subset of $X^+$. The empirical risk of a latent component $i$ aggregates the sample loss over all samples in $X^i$:
\begin{flalign} \label{eq:component_loss}
&L(W^i)=\sum_{x\in X^i} [ \frac{1}{m^+_i} L^M_{X^\textsc{-}}(W^i) + \lambda L^H_{X^\textsc{+}}(W^i;x) ] \\ \nonumber
&= L^M_{X^\textsc{-}}(W^i)+\lambda \sum_{x\in X^i}
\left[ L^H_{X^\textsc{+}}(W^i;x) \right]&& \nonumber
\end{flalign}
Finally, we define the empirical risk of the LHM model as the sum of empirical risks of all its latent components:
\begin{flalign} \label{eq:emp_risk}
&L(\wl;\varphi)=\sum_{i=1}^C{L(W^i)} \\
\nonumber&= \sum_{i=1}^C \left( L^M_{X^\textsc{-}}(W^i) \right ) + \lambda \sum_{i=1}^C \left( \sum_{x\in X^i}  L^H_{X^\textsc{+}}(W^i;x)\right) &&
\end{flalign}
By summing the risk of the components in Eq.~\ref{eq:emp_risk}, we upper bounded  the expected  minimax risk in Eq.~\ref{eq:minimax_expected} with $\sum_{i=1}^C \Pr_{z\sim Z(\mu,\Sigma)}(z \in Q^i)$.

\subsection{LHM Training}
The training aims to minimize the empirical risk in Eq. \ref{eq:emp_risk} over the parameters $\wl$ and the hidden variables $\varphi$.  Similarly to latent SVM~\cite{Yu:2009}, the complexity of the optimal assignment of samples to latent components is exponential.  We propose an iterative algorithm, which reaches fast convergence and shows good results in practice. The algorithm iterates between two steps: First, given an assignment it produces a model $\wl$, second, it updates the latent variables $\varphi(x), \forall x\in X^+$ to better represent the latent structure of the data.

The \textbf{first} step updates the LHM model $\wl^t$  in iteration $t$ given the latent variables $\varphi$ from iteration $t-1$. %Namely, positive points are split into hidden components, specified by their assignment variable.
Namely, for each hidden component $i=1,...,C$, we find the hyperplanes $W^i$ separating the training samples in $X^i$ from $D_{neg}$ by minimizing the empirical risk in Eq.~\ref{eq:component_loss}. This risk is minimized by the training algorithm proposed in~\cite{osadchy2015k}.

The \textbf{second} step updates the latent variable assignment, given the current $\wl^t$.
For each positive sample, it finds the best component w.r.t. the risk in Eq~ \ref{eq:emp_risk}. Specifically, the hinge risk for $x$ is simply $\ell(W^i,x,1)$. The minimax part of the assignment function for $x\notin Q^i$ should consider the probability that this point adds when it is included in the component $i$ (as shown in Figure~\ref{fig:wx_demo}, left). For $x\in Q^i$, the minimax part should consider the amount of probability released when the component shrinks as a result of change in the assignment of $x$ (as shown in Figure~\ref{fig:wx_demo}, right). The optimal assignment should take both cases into consideration for all components. We define the assignment as follows,
\begin{flalign}\label{eq:argmin}
&\varphi(x)= \\  \nonumber
&\argmin_{i\in\{1..C\}}\left[{\Pr_{z\sim Z(\hat{\mu},\hat{\Sigma})}(z \in Q^i_x) +\lambda\ell(W^i;x,1)}\right]&&
\end{flalign}
where $Q^i_x\triangleq \{x:{W^i_x}^Tx \geq \vec{0}\}$ and
\begin{flalign*}
&W^i_x \triangleq  \begin{cases} W^{def} &\mbox{if } x \in Q \\
                      W^{inf} &  \mbox{if } x \notin Q
        \end{cases}&&
\end{flalign*}
$W^{def}$ is a deflated model derived from $W^i$ by parallel translation of the hyperplane closest to $x$ such that $w_*^Tx+b_*=0$. $W^{inf}$ is an inflated model derived from $W^i$ by parallel translation of the hyperplanes for which $w_k^Tx+b_k<0$, until they intersect in $x$, namely, $w_k^Tx+b'_k=0$. The rest of the hyperplanes remain unchanged. The full training algorithm is summarized in Algorithm~\ref{alg:lhm}.
\begin{figure}
\center
  \includegraphics[width=0.8\linewidth]{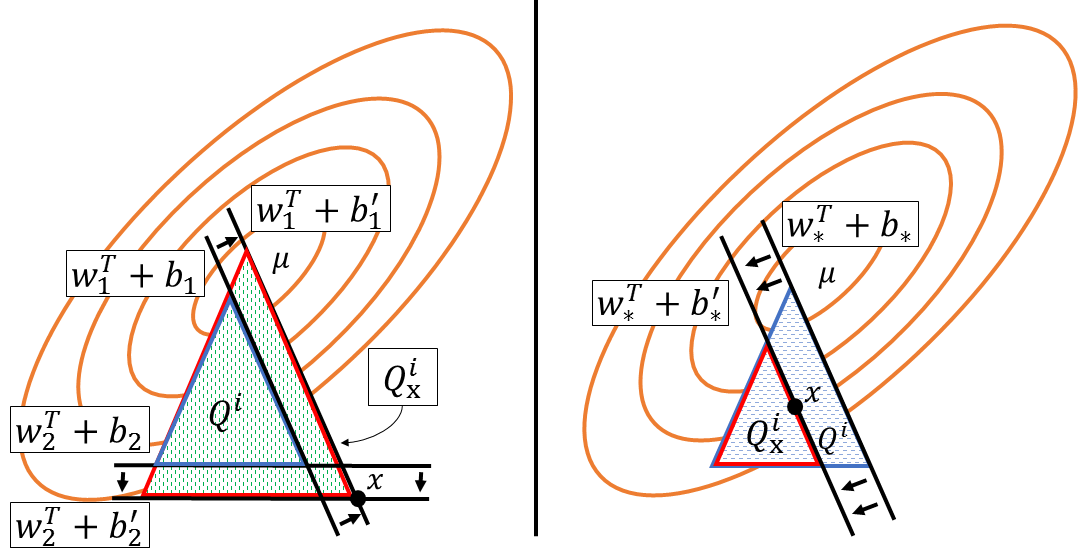}
  \caption{The orange elliptic circles represent the negative distribution $Z(\hat{\mu},\hat{\Sigma})$, the red triangles corresponds to $Q^i_x$ and  the blue ones to $Q^i$. \textbf{Left}: $w^i_1, w^i_2$ are moved to pass through $x$, causing the probability $Q^i$ to increase. \textbf{Right}:$w^i_*$ is moved to pass through $x$, causing the probability $Q^i$ to decrease.}
  \label{fig:wx_demo}
\end{figure}

\begin{algorithm}
\caption{LHM Training. \emph{KHHM-train} refers to the training of intersection of hyperplanes from~\cite{osadchy2015k}. $T$ is the threshold on the empirical risk change.}
\label{alg:lhm}
\begin{algorithmic}[1]
\INPUT{$C$, $K$, $X^+$, $X^-$,$T$}
\OUTPUT{$\wl,\varphi$}
\State $t \gets 1$
\State $L(\wl^{t=0};\varphi^{t=0}) \gets \infty$
\State $\varphi^t \gets Init(X^+,C)$  $\left\{initial\;\;\;assignment\right\}$
\Do
\ForAll{$i=1,...,C$}$\left\{Model\;\;\;Step\right\}$
\State
$W^{i,t}$=\emph{KHHM-training}($X^-$,$X^i$)
\EndFor
\ForAll{$x\in X^+$}$\left\{Assignment\;\;\;Step\right\}$
\State $\varphi^{t+1}(x)$ as defined in Eq.~\ref{eq:argmin}
\EndFor
\State $ t \gets t+1 $
\doWhile{$L(\wl^t;\varphi^t) - L(\wl^{t-1};\varphi^{t-1}) \ge T$}
\end{algorithmic}
\end{algorithm}

\begin{lemma}
Algorithm~\ref{alg:lhm} minimizes the empirical risk $L(\wl;\varphi)$.
\end{lemma}

\begin{proof}
Since LHM risk is a sum of risks over the latent components (Eq.~\ref{eq:emp_risk}), it is minimized by minimizing the empirical risk of each component. In step (5) of the Algorithm~\ref{alg:lhm}, we train $W^{i,t}$ model for each latent component $i=1,...,C$  using the iterative algorithm  from \cite{osadchy2015k} (the convergence of which was shown in \cite{osadchy2015k}). It is easy to see that $L(W^i) = L(\wkl)$, thus step (5) of the Algorithm~\ref{alg:lhm} minimizes the component's risk in Eq.~\ref{eq:component_loss}.

It is now left to show that the assignment $\varphi^t$ in iteration $t$, will cause the reduction in the empirical risk in iteration $t+1$. Since the empirical risk is aggregated over positive samples, it is enough to prove the claim for a single sample. We consider two cases:

\noindent\textbf{1.} The  assignment of sample $x$ does not change, formally $\varphi^{t}(x)=\varphi^{t+1}(x)$. In this case $L(\wl^{t+1};\varphi^{t+1}(x))$ will only be affected by the $W^{i,t+1}$ training, thus
\begin{flalign*}
&L\left(\wl^t;\varphi^t(x)\right)\ge
L\left(\wl^{t+1};\varphi^{t+1}(x)\right)&&
\end{flalign*}

\noindent\textbf{2.} The assignment of sample $x$  is changed. Formally in iteration $t$: $\varphi^{t}(x)=i$ and in interation $t+1$ exists $j \ne i$,  such that
\begin{flalign*}
\varphi^{t+1}(x)=j=
\argmin_{k\in\{1..C\}}{L^M_{X^\textsc{-}}\left(W^{k,t}_x\right)+ \lambda L^H_{X^\textsc{+}}\left(W^{k,t};x\right)}.
\end{flalign*}
Since $x \in Q^i$, reassigning it to a different component will cause  the $\Pr_{z\sim Z(\hat{\mu},\hat{\Sigma})}(z \in Q^i_x)$ to decrease (or stay the same), thus
\begin{flalign*}
&L^M_{X^\textsc{-}}(W^{i,t}_x) - L^M_{X^\textsc{-}}(W^{i,t}) \le 0.&&
\end{flalign*}
Hence, the sample loss in component $i$ is larger than the sample loss in the deflated component:
\begin{flalign}\label{eq:9}
&L\left(W^{i,t};x\right) \geq L^M_{X^\textsc{-}}\left(W^{i,t}_x\right) + \lambda L^H_{X^\textsc{+}}\left(W^{i,t};x\right).&&
\end{flalign}
At the same time, $j$ is the optimal assignment, thus
\begin{flalign}\label{eq:10}
&L^M_{X^\textsc{-}}\left(W^{i,t}_x\right) + \lambda L^H_{X^\textsc{+}}\left(W^{i,t};x\right) \ge \\ \nonumber &L^M_{X^\textsc{-}}\left(W^{j,t}_x\right) + \lambda L^H_{X^\textsc{+}}\left(W^{j,t};x\right).&&
\end{flalign}
Since $W^{j,t}_x$ is a naive inflation of $W^{j,t}$ to include $x$, the solution  $W^{j,t+1}$, provided by KHHM training, would have lower (or same) empirical risk, thus
\begin{flalign}\label{eq:11}
&L^M_{X^\textsc{-}}\left(W^{j,t}_x\right)\geq L^M_{X^\textsc{-}}(W^{j,t+1}).&&
\end{flalign}
In iteration $t+1$, $x$ is included in $X^j$ for training  the $j$'th latent component, consequently
\begin{flalign}\label{eq:12}
&L^H_{X^\textsc{+}}\left(W^{j,t};x\right)\geq L^H_{X^\textsc{+}}(W^{j,t+1};x).&&
\end{flalign}
(as we assume that $x\in X^j$ leads to $x\in Q^{j,t+1}$).
Finally, by combining the inequalities in Eq.~\ref{eq:9}--\ref{eq:12}, we obtain:
\begin{flalign*}
&L\left(W^{i,t};x\right) \geq L\left(W^{j,t+1};x\right).&&
\end{flalign*}
\end{proof}

\section{Mapping LHM Classifier to a Neural Network}
We propose to map LHM Classifier to a Neural Network. This enables 1) end-to-end training of the CNN features and LHM classifier for imbalanced problems and 2) LHM generalization to multi-class that enables using a smaller number of labeled training samples than NN.

\subsection{Binary NN}\label{sec:binary_nn}
\begin{figure}
  \includegraphics[width=\linewidth]{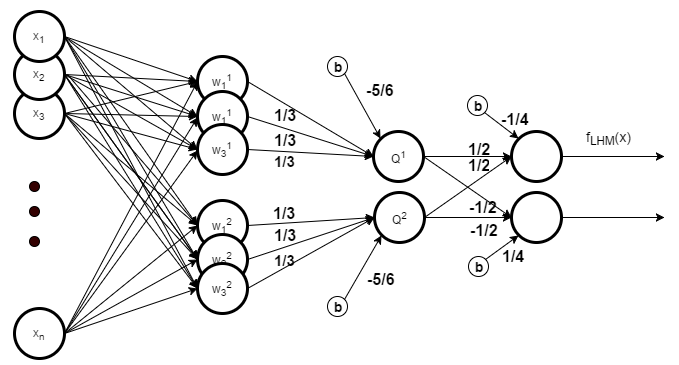}
  \caption{An example of NN equivalent to LHM.} %The input vector $x$ is projected on six hyperplanes from two latent clusters. For each latent cluster, the three projection are summed and is activated only if all are positive. On the third layer the latent clusters are summed again and activate the output if at least one of the latent clusters was activated.}
  \label{fig:LHM_NN}
\end{figure}
 A union of the intersections of positive half spaces can be implemented by a NN with three hidden layers. The first fully connected hidden layer has $K\times H$ neurons, where $K$ is the number of hyperplanes in an intersection and $H$ is the number of components. The second hidden layer has $H$ nodes, connected only to the neurons associated with hyperplanes forming the corresponding intersection. The weights on these connections and the biases are fixed and mimic \textbf{AND} operation, namely, all weights of this layer are equal to $1/K$ and the biases are equal to $-1+1/(2K)$.  The last hidden layer has two neurons, which are fully connected to the previous layer with the fixed weights and biases that mimic \textbf{OR} operation, namely, the fist neuron has  weights equal to $1/H$   and the bias of $-1/(2H)$. The second output has weights equal to $-1/H$ and the bias of $1/(2H)$.  The network has two outputs.
 An example of such network for $H=2$ and $K=3$ is depicted in Figure~\ref{fig:LHM_NN}.

\subsection{Multi-Class NN}\label{sec:mc_nn}
For a multi-class setting, we suggest to train LHM model for each class using an additional unlabeled data for estimating the statistics of the negative class. We then map these models to a multi-class NN with the following architecture. The first hidden layer  is a fully connected layer with $H\times K$ neurons per class, $H\times K\times C$ neurons in total, where $C$ is the number of classes. These are  equivalent to  $H\times K\times C$ hyperplanes in the LHM model. For each hidden component, all hyperplanes in the intersection are connected to their corresponding node in the \textbf{AND} layer (as detailed in Section~\ref{sec:binary_nn}). The \textbf{AND} layer comprises $H\times C$ neurons. The next layer is a fully connected layer, comprising $C$ nodes. The weights on the connections to the $H$ components of the corresponding class are initialized with 1's, and the weights on the remaining connections are initialized with very small values from a Gaussian distribution.  The network has $C$ outputs and is trained using the cross-entropy loss.

To provide an end-to-end training, one can consider stacking the feature extraction layers of CNN (up to fully connected layers) with one of the above networks.
\begin{figure}[t]
\center
\includegraphics[width=1.4in]{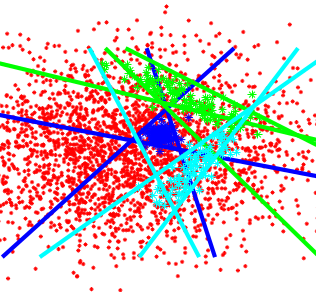}
\includegraphics[width=1.4in]{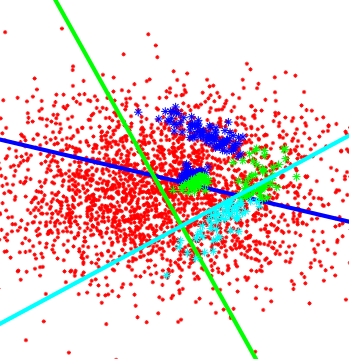}
\caption{A qualitative comparison of the latent hinge minimax classifier (on the left) to the union of LDA classifiers (on the right).}
\label{LDAvsLHM_toy}
\end{figure}

\section{Experiments}
We start by an illustrative example in 2D (Section~\ref{sec:synthetic_ex}) that shows the ability of the LHM classifier to discover the hidden components in the positive class and to separate each of them from the negative class using a K-hyperplane model.

Next, we compare LHM model to alternative ensembles of hyperplanes on the PASCAL-VOC 2007 dataset \cite{Everingham10} (Section~\ref{sec:voc_ex}), and show its advantage over those methods and its robustness to the choice of the number of latent components. In these experiments we use simple HOG features and shallow architecture.

Finally, we show (Section~\ref{sec:cifar_ex}) that LHM classifier can be combined with CNN via transfer learning.   We address two settings: 1) binary problems with imbalanced sets, 2) multi-class tasks with a small number of labeled training samples. In both cases, LHM-based models show significantly better performance than NNs. The experiments are performed on images from cifar-10 and cifar-100~\cite{cifar10} and using LeNet CNN for features extraction.
\begin{figure*}[h]
\center
\includegraphics[width=0.8\linewidth]{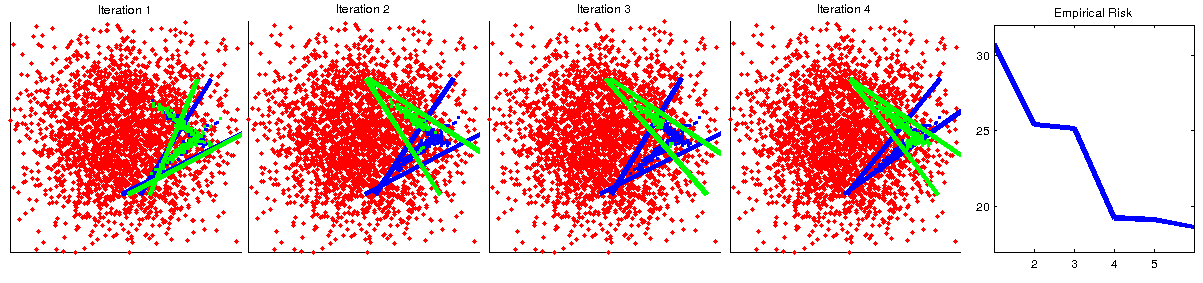}
\caption{First four iterations of the LHM training on toy example and the corresponding loss convergence. }
\label{fig:2d_example2}
\end{figure*}
\subsection{Synthetic Data}\label{sec:synthetic_ex}
A simpler alternative to the LHM model is a two-step algorithm which first finds the structure of the target class by applying some kind of unsupervised learning (e.g, k-means clustering) and then builds a model for each component. Such a simple approach was employed in~\cite{Deva} with LDA~\cite{LDA} classifier trained per cluster. Unless the clusters are very small (as in exemplar-based approach, which is time consuming~\cite{eSVM}), it relies heavily on the results of the clustering. If an initial clustering is incorrect (as in Figure \ref{LDAvsLHM_toy}, right), LDA (or any other convex classifier) cannot separate the resulting components from the background without including many false positives.  The LHM  training finds the underlying structure of the data and the model iteratively, improving both (Figure~\ref{LDAvsLHM_toy}, left). Furthermore, LHM is quite robust to the initial assignment. Figure~\ref{fig:2d_example2} shows a few iterations and the corresponding loss convergence when the  initial assignment of the positive samples to components is chosen at random. Note the LHM training discovers the underlying structure in a 3-4 iterations.

\begin{figure}[t]
  \includegraphics[width=0.8\linewidth]{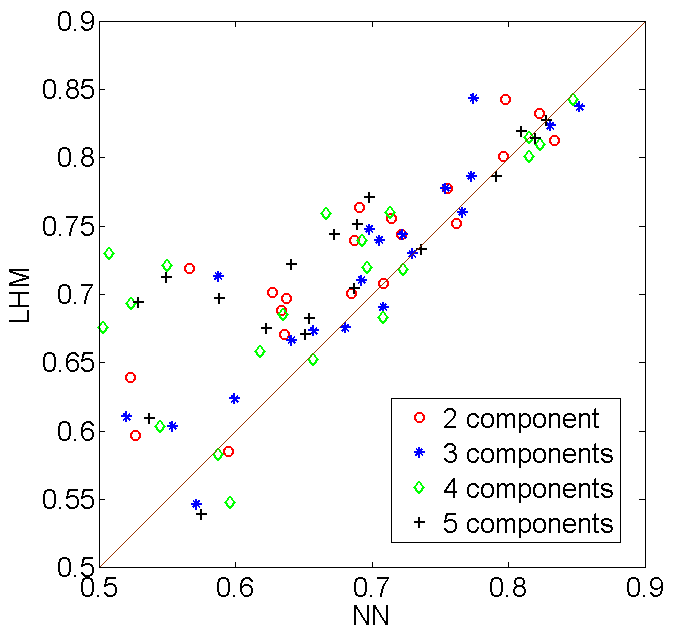}
  \caption{Comparison of the LHM classifier to the equivalent NN for a varying number of assumed hidden components (from 2 to 5) on PASCAL VOC 2007. %Comparison of the LHM classifier (the y-axis) to the equivalent NN (x-axis) for a varying number of assumed hidden components (from 2 to 5) on PASCAL VOC 2007.
  The points above the diagonal line show the advantage of LHM classifier.}\label{fig:VOC}
\end{figure}

\subsection{Ensembles of Hyperplanes}\label{sec:voc_ex}
Next, we compared the LHM classifier to alternative ensembles of linear classifiers on PASCAL VOC 2007 dataset \cite{Everingham10} using Dalal-Triggs variant of the HOG  features~\cite{HOG} with a fixed number of cells.

\noindent\textbf{LHM model:} We set the number of hyperplanes in each component to 2 and varied the number of components from 2 to 5. An initial assignment to the components was done using k-means with the Euclidian distance.

\noindent\textbf{LDA Union (as a baseline model):} We applied k-means clustering on whitened features to find the partition. We then learned an LDA classifier for each cluster in that partition. We varied the number of clusters from 2 to 5.

\noindent\textbf{NN with an architecture equivalent to LHM:} We used the model described in Section~\ref{sec:binary_nn}  with $K=2$ and $H=2,..,5$, but the weights were initialized at random.

\noindent\textbf{KHHM model~\cite{osadchy2015k}:} This is essentially an LHM model with a single component, thus it is theoretically inferior to LHM. However, we ran this experiment to test the benefits of modeling the hidden structure of the positive class. We varied the number of hyperplanes from 2 to 5.

All ensembles were trained in one-against-all manner. Similarly to~\cite{Deva,osadchy2016recognition}, we learned the background mean and covariance using bounding boxes from all classes and used them to represent the negative class in LDA union, KHHM, and LHM training.

We tested all ensemble classifiers on all windows from the test set. Table~\ref{tab:VOC} summarizes the results (\small$(1-EER)\cdot100$)) \normalsize for all tested ensembles averaged over classes and different parameters. It shows that LHM model outperforms all other classifiers. Figure \ref{fig:VOC} compares LHM to NN on 20 categories (as one-against-all binary classifiers) for varying number of hidden components. The plot shows that LHM outperforms NN independently of the number of components.

\begin{table}
\center
\begin{tabular}{|c|c|c|c|}
\hline
 LHM & Union of LDAs & NN & KHHM\\
  \hline
   71.48\%& 65.17\% &67.19\%&69.45\% \\
\hline
\end{tabular}
\caption{The table reports (1-EER)*100 averaged over 20 classes and different hidden partitions (except for KHHM) on PASCAL VOC-2007 classification task using 80-dimensional HOG features.}\label{tab:VOC}
\end{table}

\begin{figure*}
\center
  \includegraphics[width=0.37\linewidth]{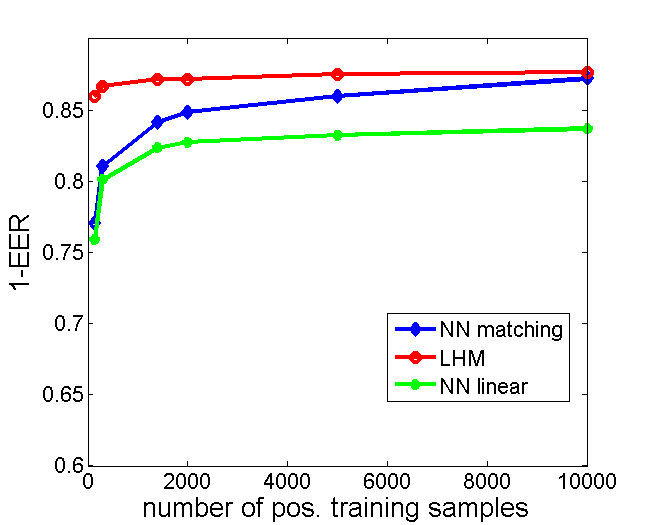}
  \includegraphics[width=0.37\linewidth]{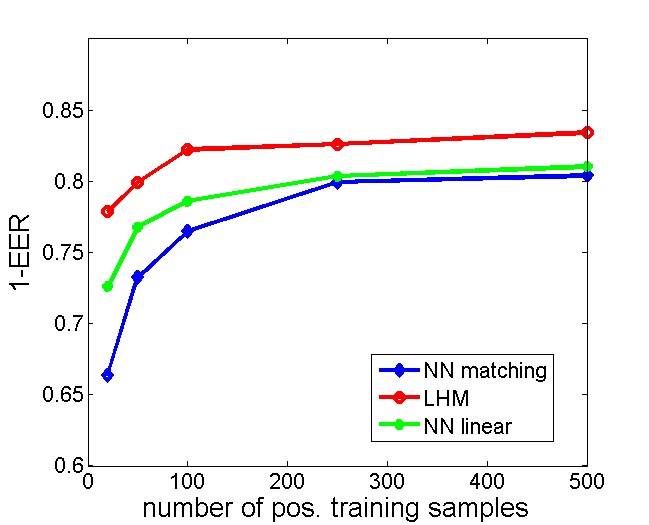}
  \caption{Binary imbalanced classification: left -- the ``best-case'' transfer learning setting, right -- the ``worst-case'' transfer learning  setting.} %The plots show the average (over 5 binary problems) 1-EER for LHM and NN as a function of positive training size: 140, 300, 1400, 2000, 1000(all) samples.} The accuracy of NN is averaged over 50 random initializations.}
  \label{fig:binary}
\end{figure*}
\subsection{Deep Architecture}\label{sec:cifar_ex}
Next, we tested the LHM classifier on top of the pre-trained CNN feature extraction in imbalanced binary problems and in multi-class tasks with a small number of labeled examples. We explored the following transfer learning settings. The first setting refers to the \textbf{best case} scenario in which the source and the target classification tasks operate on the \emph{same} set of features but differ in the classification problem. The second setting refers to the \textbf{worst case} scenario for the transfer learning where the source and the target classification problems \emph{share very little similarity}.  The ``worst case'' scenario is very common in practice, as many classification tasks do not have a large, comprehensive training set (such as ImageNet \cite{imagenet_cvpr09} in object recognition) to be used in transfer learning. No good solution currently exists for such problems.

We used the CIFAR-10, composed of 10 categories (airplane, automobile, bird, cat, deer, dog, frog, horse, ship, and truck) as the source problem. Specifically, we trained the LeNet model implemented in MatConvNet~\cite{vedaldi15matconvnet} on CIFAR-10. Then we removed the last fully-connected layer and the soft-max and used this trimmed network as a feature extractor which converts images to a $64$-dimensional feature vectors.

For the best case transfer learning, we defined a new set of classes by coupling $i$ and $i+5$ indexes of CIFAR-10 classes. CNN trained on CIFAR-10 maps individual classes to linearly separable sub-spaces, thus using pairs of classes as a target classification problem makes it non-linear. Consequently, we get a new classification problem over the same space of features.

For the worst case transfer learning, we picked a subset of 5 classes (train, bottle, cattle, forest, and sweet peppers) from the CIFAR-100, which do not overlap (in their visual appearance) with the CIFAR-10 categories, to be the target classification task. CIFAR-10 data set is not rich enough to enable learning of features that can be used for an arbitrary category, thus we believe that such setting is especially difficult.

We tested the LHM binary and multi-class classifiers in the best and the worst case transfer learning scenarios and compared their performance to two baselines. One is an NN with a single fully connected  layer and the cross-entropy loss (NN linear) and the other is the NN with the architecture matching the LHM model (NN matching).  We repeated each experiment 50 times over different random subsets of training samples and random initialization of NN and averaged the results.

\begin{figure*}
\center
  \includegraphics[width=0.37\linewidth]{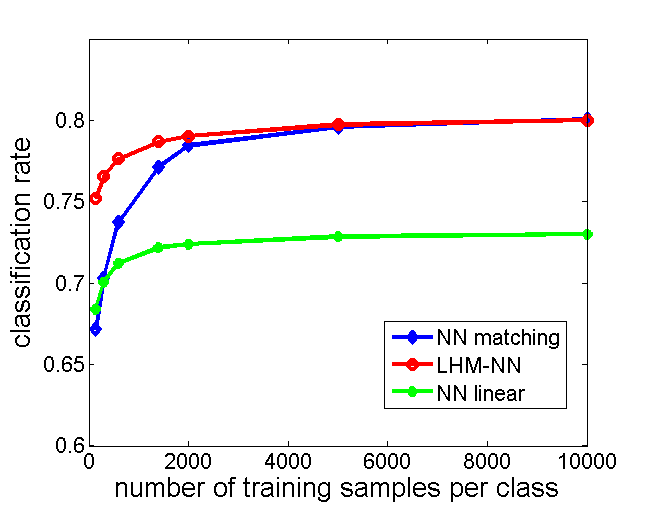}
  \includegraphics[width=0.37\linewidth]{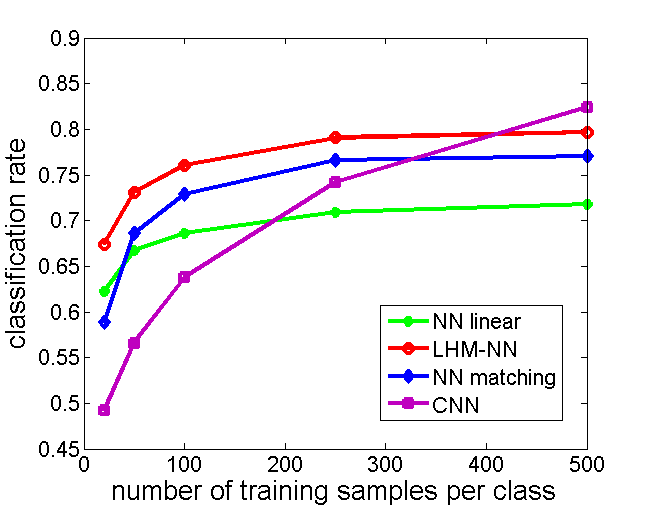}
  \caption{Multi-class classification: left -- the ``best-case'' transfer learning  setting, right -- ``worst-case'' transfer learning  setting.}
  \label{fig:mc}
\end{figure*}

\subsubsection{Binary Imbalanced Problem}
\noindent\textbf{The ``Best Case'' Transfer Learning:}
We trained binary classifiers for pairs of classes from CIFAR-10 using imbalanced training sets, in which the negative class included all samples from all other classes (40,000 examples) and the positive class included a varying number of samples (140, 300, 600, 1400, 2000, 5000-all). This resulted in imbalance ratios from 1:256 to 1:4.

LHM model was trained with 2 hidden components and 3 hyperplanes per component. The matching NN mimicked the configuration of LHM model, but the weights were allowed to change in training.  Figure~\ref{fig:binary}-left shows the 1-EER (averaged over 5 classification problem) of the LHM classifier and the two NN baselines as a function of the positive training sample size.

\noindent\textbf{The ``Worst Case'' Transfer Learning:}
Since the number of samples per class in CIFAR-100 is significantly smaller, this experiment tests the robustness to imbalanced training data and to a small number of examples. We varied the size of the positive training set between 20, 50, 100, 250, 500(all) samples and we used all 2,000 samples of other classes as the negative training set. We compared the LHM model trained with 2 hidden components and  2 hyperplanes per component to NN baselines. Figure~\ref{fig:binary}-right shows the 1-EER of the classifiers averaged over 5 classification problems as a function of the positive training set size.

%In both setting the LHM model showed much better robustness to the size of the positive training set than the NN with matching architecture.

\subsubsection{Multi-Class Problem}
\noindent\textbf{The ``Best Case'' Transfer Learning:}
We mapped the LHM binary classifiers trained for 5 pairs of categories to a multi-class NN as described in Section~\ref{sec:mc_nn}. We fine-tuned the weights with a very fast training (just a handful of epochs, while training from scratch requires two orders of magnitude more training epochs). Figure~\ref{fig:mc}-left shows the accuracy of the LHM models mapped to a multi-class NN (LHM-NN) with the two baseline NNs as a function of the size of the training set.

\noindent\textbf{The ``Worst Case'' Transfer Learning:}
We mapped the LHM binary classifiers trained for the 5 categories from CIFAR-100 (using CIFAR-10 features) to a multi-class NN and fine-tuned the weights with a small number of epochs.

To test the complexity of the transfer learning problem we also trained a CNN (LeNet model implemented in MatConvNet~\cite{vedaldi15matconvnet}) on the target problem. We hoped that due to the small size of the target classification problem, 500 training examples per class would  yield relatively good accuracy. Figure~\ref{fig:mc}-right compares the accuracy of LHM-NN, two baseline NNs, and CNN (trained from scratch) as a function of the training sample size. It shows that CNN trained on the target problem is indeed the best as it succeeds to learn features specific for the task, but its accuracy drops very abruptly when the number of training samples becomes smaller. This suggests that when the number of training examples is small, using transfer learning even in a such difficult setting is a better solution than training a CNN from scratch.

The results in Figures \ref{fig:binary} and \ref{fig:mc} show that the NN models either heavily overfit when the number of  training samples is small (NN matching) or they are not expressive enough when the number of training samples increases (NN linear). LHM classifiers are expressive enough to learn from a large set of examples and are more robust to overfitting when the number of examples is small.

%Our experiments for multi-class problems show that LHM-NN is the best performing classifier when the number of labeled samples is small.  %The benefit of LHM-NN can be explained by its ability to incorporate unlabeled data into the training.

\section{Training Efficiency}
Another advantage of LHM-NN is its training efficiency. A class-specific LHM model converges in 5-10 iterations. Its training time primarily depends on the number of positive samples and the dimension. The negative samples are used to estimate the mean and covariance of the background. The initial estimation (which involves a large number of samples) can be done only once and used for all classes. Since the probability of the negative class is evaluated inside the positive region using false positives~\cite{osadchy2015k}, the number of which drops very fast, the estimation time of the mean and covariance during the training is negligible. Training of a binary classifier per class is independent of other classes, thus their training can be done in parallel. Finally, the fine-tuning of the multi-class network after mapping is very fast, due to the initialization of all layers (using supervized learning): feature extraction layers with pre-trained CNN and classifier's layers with LHM models.

The LHM-NN is also beneficial for the problems in which classes are dynamically added or removed from the classification task. Adding a class requires training a single binary classifier and fast fine-tuning; removing a class requires only fine-tuning.

\section{Conclusions}
We proposed a novel Latent Hinge-Minimax classifier for binary problems that discovers the hidden components in the positive class and separates them from the negative class with the intersections of positive half spaces. The main advantage of this classifier is its ability to incorporate unlabeled data in training. This results in a better robustness to imbalanced problems. We showed that for multi-class tasks, class-specific LHM models can be mapped to a multi-class NN with matching architecture requiring only a few iterations of fine-tuning. Finally, the proposed LHM architecture can be integrated with CNN features via transfer learning. The entire training procedure is very efficient. Our experiments showed that such classifiers are much more robust to the number of labeled training samples than the equivalent NNs.

We plan to incorporate multi-class loss into the Hinge-Minimax paradigm and design an efficient algorithm for minimizing this risk. We also plan to tain NNs using Hinge-Minimax like loss. It would be interesting to compare the results of this model with the binary-to-multiclass mapping scheme proposed here.
%\section{Acknowledgments}
%This work has been supported by Israel Science Foundation
%839/12.

{\small
\bibliographystyle{ieee}

\end{document}